\documentclass[10pt,journal,compsoc]{IEEEtran}
\usepackage{amsmath}
\usepackage{amssymb}
\usepackage{amsfonts}
\usepackage{amsthm}

\usepackage{graphicx}
\usepackage{bm} % optional
\usepackage{multirow}
\usepackage{float}
\usepackage{xcolor}
\usepackage{comment}
\usepackage{animate}

\ifCLASSOPTIONcompsoc
  \usepackage[nocompress]{cite}
\else
  \usepackage{cite}
\fi
\ifCLASSINFOpdf
\else
\fi
\newtheorem{theorem}{Theorem}
\newtheorem{lemma}[theorem]{Lemma}
\newtheorem{corollary}[theorem]{Corollary}

\begin{document}
\title{EBSnoR: Event-Based Snow Removal by Optimal Dwell Time Thresholding}

\author{Abigail~Wolf,~\IEEEmembership{Student Member,~IEEE,}
        Shannon~Brooks-Lehnert, %~\IEEEmembership{Member,~IEEE,}
        and~Keigo~Hirakawa,~\IEEEmembership{Senior~Member,~IEEE}% <-this % stops a space
\IEEEcompsocitemizethanks{\IEEEcompsocthanksitem A.~Wolf and K.~Hirakawa are with the Department of Electrical and Computer Engineering, University of Dayton, Dayton, OH 45469.\protect\\
E-mail: \{wolfa8,khirakawa1\}\@udayton.edu.
\IEEEcompsocthanksitem S.~Brooks-Lehnert is with Ford Motor Company.}% <-this % stops an unwanted space
\thanks{Manuscript received April 19, 2005; revised August 26, 2015.}}


\markboth{Journal of \LaTeX\ Class Files,~Vol.~14, No.~8, August~2015}%
{Shell \MakeLowercase{\textit{et al.}}: Bare Demo of IEEEtran.cls for Computer Society Journals}



\IEEEtitleabstractindextext{%
\begin{abstract}
We propose an \underline{E}vent-\underline{B}ased \underline{Sno}w \underline{R}emoval algorithm called EBSnoR. We developed a technique to measure the dwell time of snowflakes on a pixel using event-based camera data, which is used to carry out a Neyman-Pearson hypothesis test to partition event stream into snowflake and background events. The effectiveness of the proposed EBSnoR was verified on a new dataset called UDayton22EBSnow, comprised of front-facing event-based camera in a car driving through snow with manually annotated bounding boxes around surrounding vehicles. Qualitatively, EBSnoR correctly identifies events corresponding to snowflakes; and quantitatively, EBSnoR-preprocessed event data improved the performance of event-based car detection algorithms. 
\end{abstract}

}

\maketitle

\IEEEdisplaynontitleabstractindextext

% \fi
\IEEEpeerreviewmaketitle

\IEEEraisesectionheading{\section{Introduction}\label{sec:introduction}}

\IEEEPARstart{I}{n} automotive imaging applications, environmental noise such as rain, snow, sun flares, etc.~are nuisances that significantly deteriorate the performance of driver assistive computer vision algorithms. For instance, frames operating at conventional 30 or 60 frames per second yield snowflake streaks that obstruct the view of the road and the surrounding vehicles that make it difficult for computer vision algorithms to detect them accurately. This can be overcome to an extent by speeding up the frame rate to slow the snowflakes or reduce the frame integration time to freeze the snow. However, it comes at the significant computational burden of the post processing computer unit to accommodate higher frame throughput, or reduced sensitivity to light and increased noise.

We propose an event-based snow removal algorithm called EBSnoR, aimed at partitioning the event stream to snowflakes and background events as shown in Figure \ref{fig:teaser}. EBSnoR exploits spatial-temporal statistical constraints of the snow as it appears on the image formed at the detector. Event-based cameras are better suited to detect snowflakes than the conventional intensity cameras because of their inherent asynchronous readout circuit and the high temporal resolution. Consider Figure \ref{fig:accumulations} where a synthetic frame is rendered by accumulating events at 1ms, 16.66ms, and 33.33ms time intervals. The 1ms frame demonstrates the ability for event-based sensors to resolve high speed phenomenon, as evidenced by the individual snowflake particles observed. Comparing the 1ms frame with the 16.66ms and 33.33ms frames, the latter frames (corresponding to the integration times of 60 and 30 frames per second used by the conventional intensity cameras) suffer from long snowflake streaks stemming from the vehicle moving towards them at high speeds.

\begin{figure}
\begin {center}
\animategraphics[autoplay, loop, width=9cm]{12}{figures/city_gif/frame_}{1}{100}
\end{center}
\caption{EBSnoR partitions event stream into snowflakes (red) and background (green=positive, blue=negative) events. Applied to UDayton22EBSnow ``City'' sequence and played back at 0.3$\times$ speed.}
\label{fig:teaser}
\end{figure}

Owing to the real-time nature of automotive applications, we focus on simple and effective technique for snow removal. For instance, we rule out the possibility of tracking individual snowflakes whose complexity scales with the density of the snow and becomes unmanageable in heavy precipitation. Instead, the proposed EBSnoR is formulated as a statistically optimal likelihood ratio test performed on the snowflake dwell time, or the duration of the time that a snowflake is observed by a particular pixel. In Section \ref{sec:dwell-time-model} we derive the probability density function (pdf) of the snowflake dwell time.  In Section \ref{sec:removal-model} we develop an event-based method to measure the snowflake dwell time and propose a Neyman-Pearson hypothesis test-based snowflake detection via the dwell time thresholding. In Section \ref{sec:experiments}, we demonstrate that the performance of the event-based car detection algorithm improves when performed on the background events only.
\begin{figure*}
\begin{center}
    \begin{tabular}{ccc}
         \includegraphics[width=0.25\textwidth]{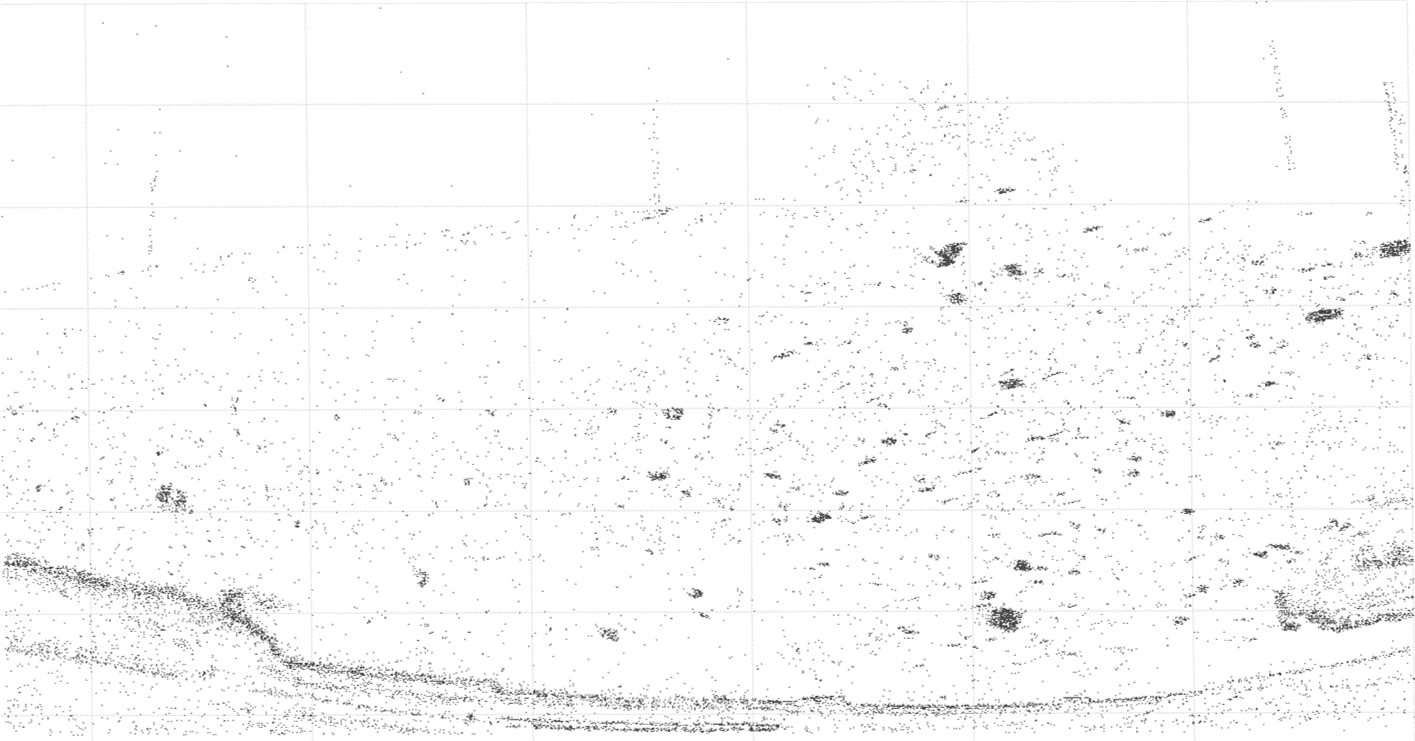} & 
         \includegraphics[width=0.25\textwidth]{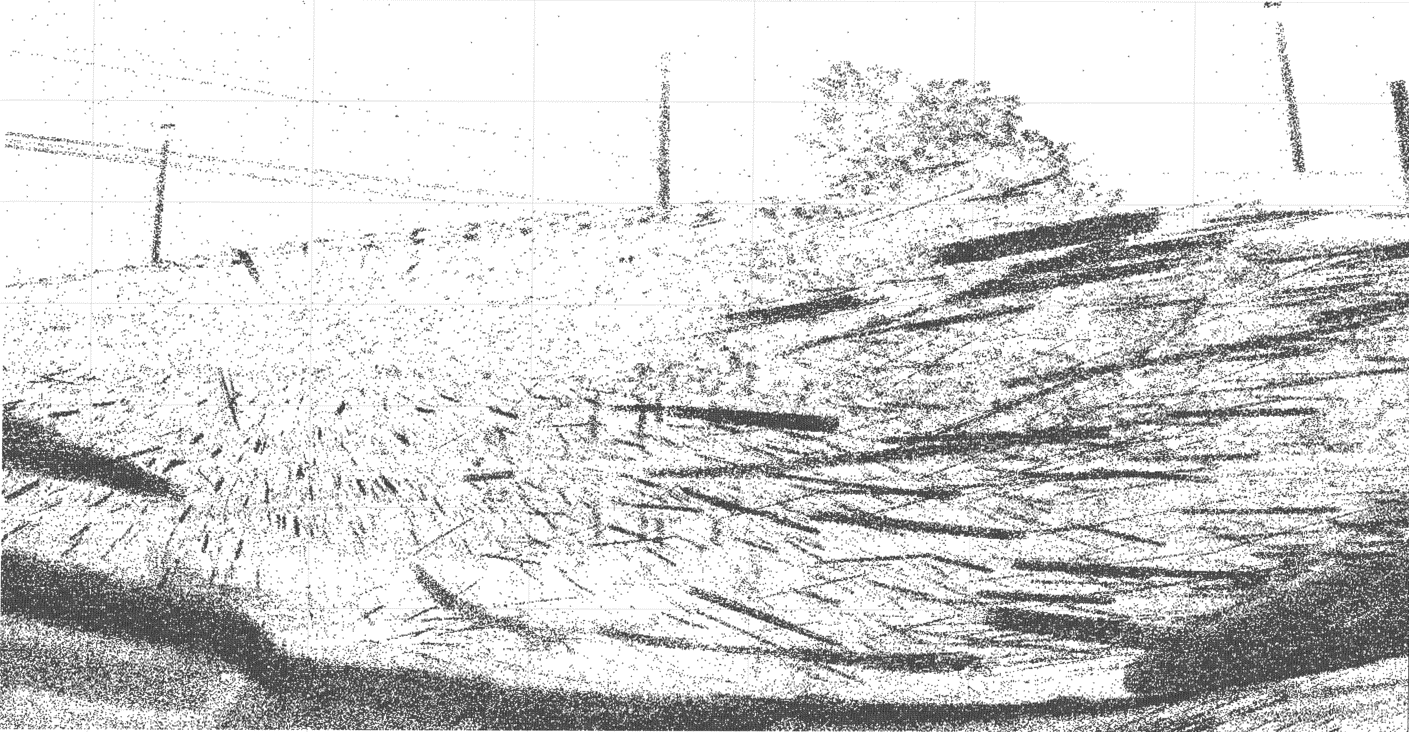} &
         \includegraphics[width=0.25\textwidth]{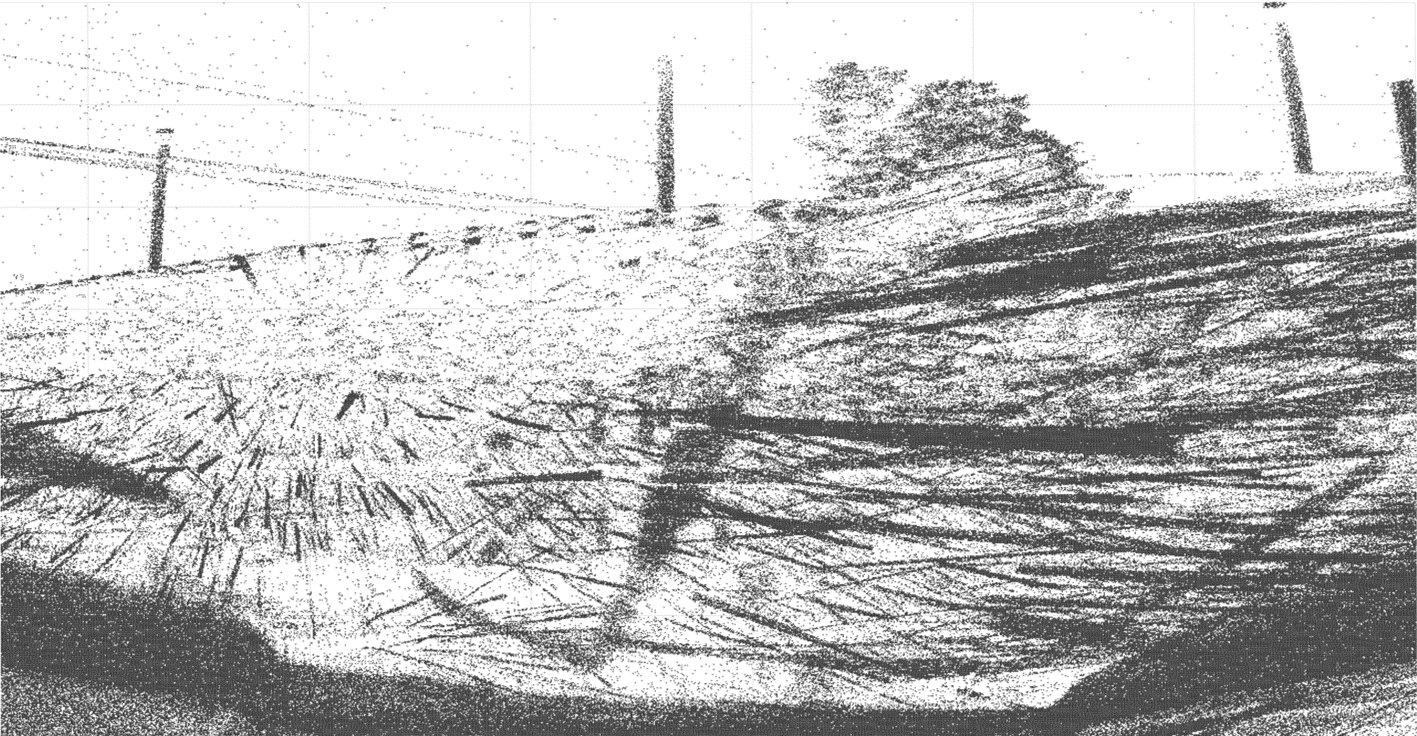} \\
         (a) 1ms & (b) 16.66ms & (c) 33.33ms
    \end{tabular}
\end{center}
\caption{Events from UDayton22EBSnow dataset accumulated over 1ms, 16.66ms, and 33.33ms time windows. The latter two accumlatations correspond to typical framerates of conventional intensity cameras of 60 and 30 frames per second, respectively, yielding snowflake streaks.}
\label{fig:accumulations}
\end{figure*}

Our contributions can be summarized as follows:
\begin{itemize}
    \item A rigorous statistical analysis of snowflake dwell time and the development of Neyman-Pearson hypothesis test.
    \item EBSnoR: event-based dwell time measurement technique used to carry out the likelihood ratio test to partition the snowflake and the background events.
    \item UDayton22EBSnow: event-based camera sequences captured from a moving vehicle driving through snow in a front-facing camera configuration, with manual bounding-box annotations of surrounding cars.
\end{itemize}

\section{Background and Related Work}

\subsection{Event Cameras}
Event-based camera is an emerging imaging modality modeled after the human visual system. In contrast to a conventional camera with active pixel sensors (APS) designed to measure pixel intensity synchronously across a frame, pixels in event-based sensors detect intensity changes asynchronously. Specifically, each pixel operates independently of the surrounding pixels to detect a change in log-intensity exceeding a threshold. This change detection is reported as an ``event,'' coupling the X-Y pixel coordinate and the polarity of the event (whether the event is caused by an increase or a decrease in intensity) with the event timestamp.

The advantages of the event cameras over the conventional intensity cameras include the dynamic range exceeding 120dB, temporal resolution in microseconds, 10-100 order millisecond latency, and power consumption in 30 milliwatt range. Owing to its potential to advance the state of real-time high speed image processing and computer vision systems, event-based applications have been considered in fields such as camera tracking \cite{cam-tracking}, high-speed/high-dynamic-range video \cite{high-speed-vid}, machine learning \cite{machine-learning}, image reconstruction \cite{image-recon}, and autonomous vehicles \cite{autonomous-vehicles}.

\subsection{Inceptive Event Filtering}
\label{sec:IEF}

Multiple events are generated when a pixel encounters an intensity change or an ``edge.'' These events are triggered sequentially rather than simultaneously, introducing ambiguity in timing. Inceptive event filtering is a method of organizing event stream in a way to better reflect the edge properties\cite{Baldwin_2019}. Specifically, inceptive event filtering categorizes events into three types: inceptive events (IE), trailing events (TE), and noisy events (NE). \emph{Inceptive event} refers to the first of the series of events generated from an edge encounter, which corresponds to the exact timing of the edge arrival. The subsequent events of the same polarity are referred to as the \emph{trailing events}. The trailing events are proportional in number to the magnitude of the log-intensity edge, but the latency associated with them make the event timestamps ambiguous and imprecise. Thus we often count and report the number of trailing events associated with a given inceptive events as ``edge magnitude,'' but not their event timestamps. Inceptive event with no corresponding trailing events are sometimes referred to as \emph{noisy events} and removed from event stream when not useful. 

Baldwin \emph{et al.} performed the event classification based on the heuristics shown in Table \ref{tab:IETS-params}\cite{Baldwin_2019}. Here, $event(n).ts$ represents the timestamp of the event currently being processed, $event(n-1).ts$ represents the timestamp of the previously processed event, $event(n+1).ts$ represents timestamp of the next event to be processed, and $\Delta$ is a predefined threshold. Steps to carry out the inceptive event filtering are summarized in Figure \ref{fig:state-diagrams}(a). The output may be just an event label (IE, TE, or NE), or graph as shown in Figure \ref{fig:state-diagrams}(b). 

Example output from inceptive event filtering is shown in Figure \ref{fig:state-diagrams}, showing only inceptive events. Since the inceptive events correspond to the exact timing of the edges, the inceptive events simplify the overall image while maintaining high edge shape fidelity. 

\begin{table*}
\caption{Classification parameters for inceptive event filtering algorithm \cite{Baldwin_2019}} 
\label{tab:IETS-params}
\begin{center}       
\begin{tabular}{|c||c|c|}
\hline
\rule[-1ex]{0pt}{3.5ex}
& event(n).ts - event(n-1).ts $\leq \Delta$ & event(n).ts-event(n-1).ts $> \Delta$  \\
\hline
\hline
\rule[-1ex]{0pt}{3.5ex} event(n+1).ts-event(n).ts $\leq \Delta$ & Trailing Event & Inceptive Event   \\
\hline
\rule[-1ex]{0pt}{3.5ex} event(n+1).ts-event(n).ts $ > \Delta$ & Trailing Event & Noisy Event  \\
\hline 
\end{tabular}
\end{center}
\end{table*}

\begin{figure*}
\begin{center}
    \begin{tabular}{cc}
         \includegraphics[width=0.6\textwidth]{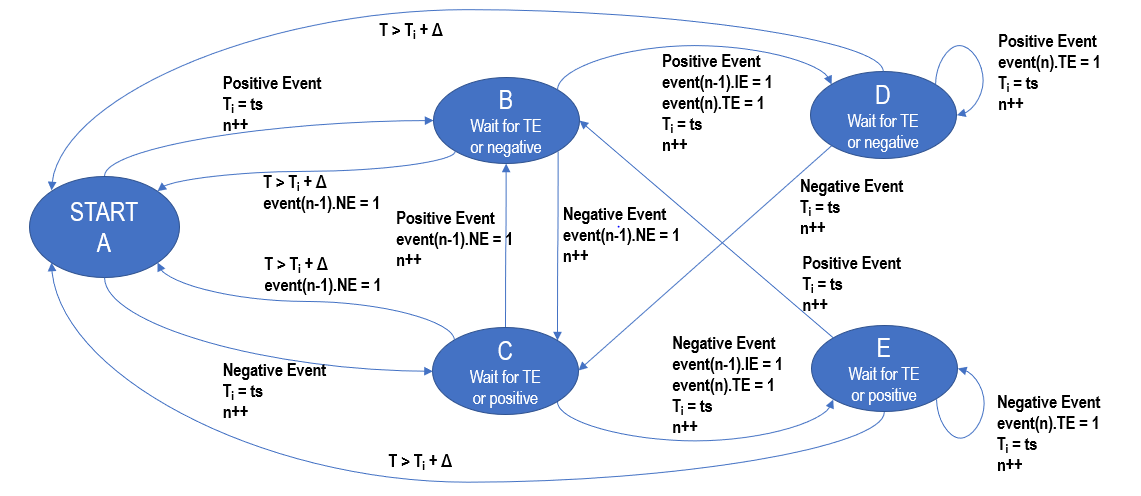} &
         \includegraphics[width=0.3\textwidth]{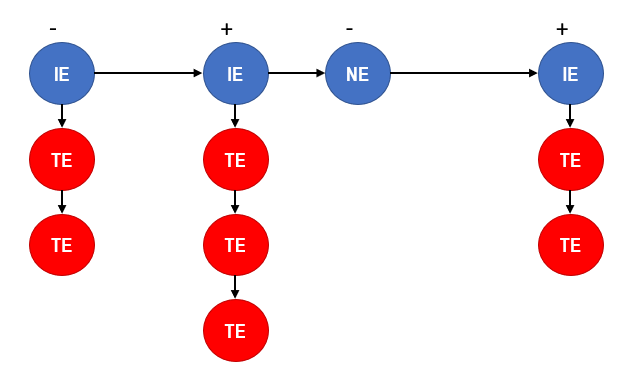}  \\
        (a) Inceptive event filtering state diagram & (b) Output inceptive event graph
    \end{tabular}
\end{center}
\caption{State diagram and output graph structure of inceptive event filtering.}
\label{fig:state-diagrams}
\end{figure*}

\begin{figure*}
    \centering
        \includegraphics[width=0.75\textwidth]{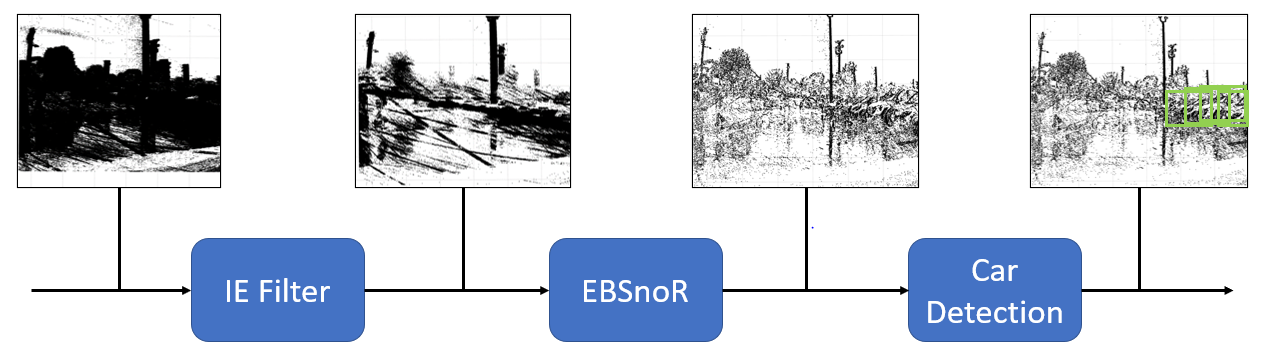}
    \caption{System diagram and data visualization of the event processing including the proposed event-based snow removal (EBSnoR) algorithm.}
    \label{fig:algo-function}
\end{figure*}

% \begin{figure}[ht]
%     \centering
%     \begin{tabular}{cc}
%         \includegraphics[width=4cm]{figures/pre_iets.png} &
%         \includegraphics[width=4cm]{figures/post_iets.png} \\
%         (a) Pre-IE filtering & (b) Post-IE filtering
%     \end{tabular}
%     \caption{Caption}
%     \label{fig:iets-comparison}
% \end{figure}

%%%%%%%%%%%%%%%%%%%%%%%%%%%%%%%%%%%%%%%%%%%%%%%%%%%%%%%%%%%%%%%%%%%%%%%%%%%%%%

\subsection{Image Snow Removal Techniques}

Techniques aimed at mitigating environmental noise such as rain and snow have been developed for conventional intensity cameras generating frames. Since intensity cameras fully capture background scenery in addition to foreground snow/rain, image-based rain/snow removal methods typically require two major steps: detection of rain/snow-affected pixels, and replacement of those pixels with estimated background intensity values. Existing detection methods typically fall into four main categories: photometric, geometric, temporal, and machine learning. 

Photometric methods of rain/snow removal rely on the visual properties of snow, such as brightness, saturation and color. Methods by Chen \emph{et al.} \cite{photo-chen} and Xu \emph{et al.} \cite{photo-xu} use single guided filters to separate high and low spatial frequency components as a way to process rain/snow removal. Zheng \emph{et al.} \cite{photo-zheng} also make use of guided filters, however their algorithm uses multi-guided filters in order to separate first the high and low frequency components of the image, and then to separate the rain/snow from the background in the high-frequency data. The method proposed by Manu \cite{photo-manu} uses $L_0$ gradient minimization to locate and preserve significant edges, treating rain/snow as a form of noise. A method by Santhaseelan and Asari \cite{photo-santhaseelan} uses the phase congruency of rain to detect potential rain streaks in video. Pei \emph{et al.} \cite{photo-pei} make use of the saturation and visibility properties of rain/snow, developing a network of high pass filters, orientation filters, and thresholding that will isolate rain/snow streaks. Garg and Nayer \cite{photo-garg} developed a rain removal technique leveraging a physics-based model to describe the way that rain blurs an images. A K-means clustering approach proposed by Zhang \emph{et al.} \cite{temp-zhang} use intensity histogram across the entire video sequence to cluster and profile snow/rain profiles separate from the background.

Geometric rain/snow removal methods rely on the patterns and movement properties of rain/snow streaks. Bossu \emph{et al.} \cite{geo-bossu} propose a method by which potential rain/snow streaks are identified using a histogram of orientation of streaks. A similar method is used by Brewer and Liu \cite{geo-brewer}, identifying pixels that exhibit a very short-term intensity spike matching the shape properties of rain drops. Another method proposed by Li \emph{et al.}\cite{geo-li} posits that the rain patterns are repetitive and sparsely scattered, and use these attribute to develop a multiscale convolutional sparse coding model that extract the streak patterns. Ren \emph{et al.} \cite{geo-ren} approach the problem from a matrix decomposition point of view classify intensity fluctuations caused by background, foreground, sparse rain/snow, and dense rain/snow. Barnum \emph{et al.} \cite{geo-barnum} develop a physical model of a raindrop/snowflake in order to determine the general shape and brightness of streaks, which is combined with the statistical properties of rain/snow to identify and extract streak patterns.

Temporal rain/snow removal methods inherently make use of multiple video frames. Kim \emph{et al.} \cite{temp-kim} observe that rain/snow particles are too small to affect optical flow of an image, and thus obtain a rain map by subtracting temporally warped frames from the current processing frame. The map is then decomposed into a sparse representation to classify pixels as rain or not rain using support vector machine. 

In recent years, removal of rain from intensity camera images/videos has moved to the machine learning realm. Porav et al \cite{ml-porav} take a denoising approach, and use a generator model to remove rain/snow noise from images. Zhang and Patel \cite{ml-zhang} proposed to use a densely connected convolutional neural network architecture to improve the  snow removal performance. Jiang \emph{et al.} \cite{ml-jiang} use Langrangian shrikage algorithm to discriminate \emph{a priori} models of rain-affected and non-affected images.

Separate from the task of detection, the \emph{removal} of rain/snow requires replacing the affected pixel values. In images, this is commonly be accomplished by inferring the proper intensity and color of a rain-affected pixel through the analysis of nearby background objects \cite{photo-chen}. In video, this is commonly performed by averaging the background color values at the same pixel from later or earlier frames\cite{temp-zhang} \cite{photo-garg} \cite{geo-brewer}.

Rain/snow removal has been studied in automotive settings also. De Charette \emph{et al.} \cite{De-2012-120326} developed a smart headlight system that will de-illuminate raindrops and snowflakes detected by a camera. This will create the effect of the driver seeing \emph{through} the rain/snow to observe the illuminated background.

The rain/snow removal task for event-based cameras is profoundly different from the intensity-based camera methods described above.  Owing to asynchronous pixel readout circuit with microsecond-order temporal resolution, the generated events track the raindrop and snowflake movements continuously. %, as evidenced by Figure \ref{fig:snow-polarity}. 
As shown by events accumulated over 16.66ms and 33.33ms in Figure \ref{fig:accumulations}, the snowflakes in a driving scenario can travel considerable distances between the consecutive frames of typical intensity cameras. As such, event-based cameras have a significant advantage in detecting and removing snowflakes.

Another key difference is that in the event-based snowflake removal task, we simply remove events corresponding to snowflakes without estimating the ``missing'' background events. As is shown by our analysis in Section \ref{sec:dwell-time-model}, the dwell time for a snowflake on a particular pixel is in the order of milliseconds. Hence only a handful of background events would actually be blocked by snowflakes in practice.

%Due to the unique way that event cameras encode data, we propose a time/polarity-based method of snow detection. Furthermore, as event cameras only capture the edges of moving objects, removal of these detected tracks is as simple as removing their data from the larger data stream. However, the process of snow detection and removal is fundamentally different for intensity cameras. 

%%%%%%%%%%%%%%%%%%%%%%%%%%%%%%%%%%%%%%%%%%%%%%%%%%%%%%%%%%%%%%%%%%%%%%%%%%%%%%

\section{Snowflake Dwell Time Modeling} \label{sec:dwell-time-model}

\subsection{Analysis of Snowflake Dwell Time}
\begin{figure}
   \begin{center}
   \includegraphics[height=3.5cm]{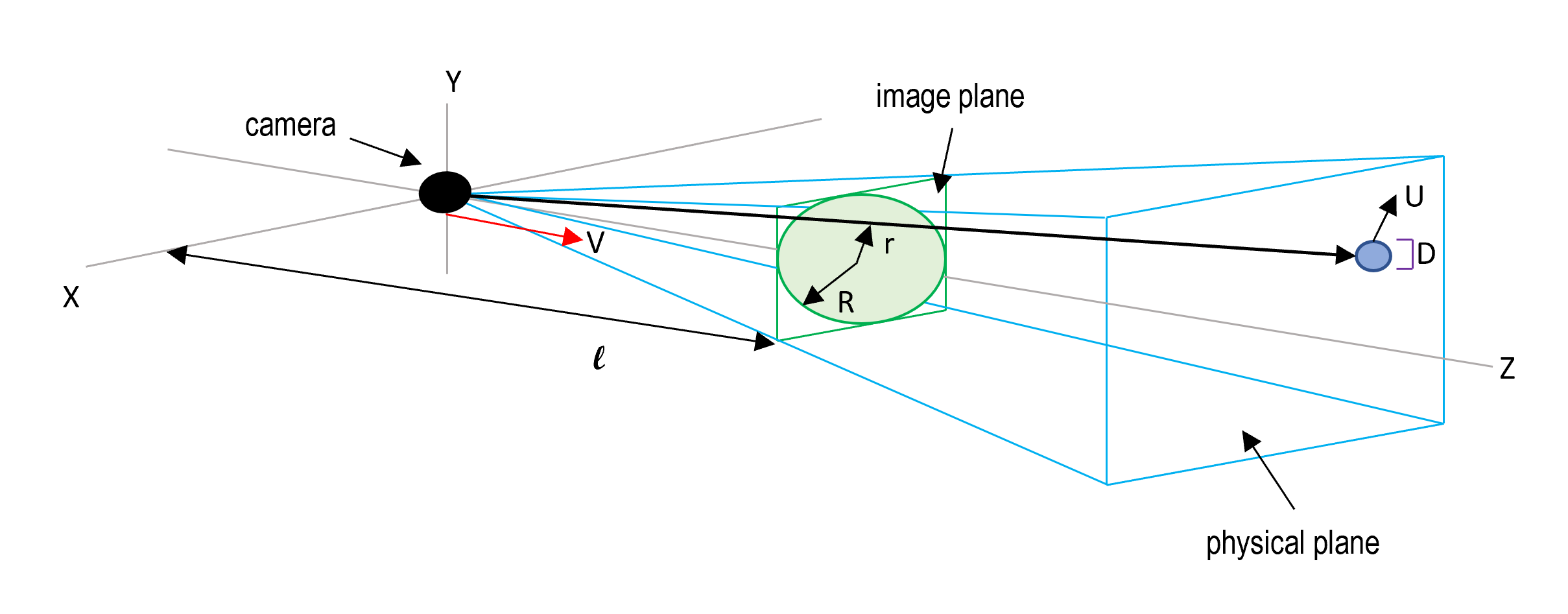}
   \end{center}
   \caption{Imaging coordinate system used in Section \ref{sec:dwell-time-model}.}
   \label{fig:math} 
\end{figure}

As illustrated in Figure \ref{fig:math}, assume that the car is moving along the optical axis $Z$ of the front-facing camera at the speed $V$ far exceeding the snowflake's velocity (i.e.~snowflake velocity is negligible). Then, the snowflake projected onto the two dimensional image of the front-facing camera appears to be moving as the moving vehicle approaches the snowflake. Define $T$ as the ``dwell time'' or the amount of time that a snow flake intersects a particular pixel. We prove in the {\bf Lemma \ref{lem:dwell_time}} below that the snowflake's dwell time is independent of the distance $Z$ along the optical axis.

\begin{lemma}
\label{lem:dwell_time}
Let $(x,y)\in\mathbb{R}^2$ be the image plane coordinate of the front-facing camera at focal length $\ell$. Then the dwell time of the snowflake when observed by the camera is
\begin{align}\label{eq:dwell}
    T=\frac{D\ell}{V\sqrt{x^2+y^2}},
\end{align}
where $V$ is the vehicle velocity and $D$ is the diameter of the snowflake.
\end{lemma}
\begin{proof}
The apparent velocity $U$ of the snow due to the vehicle motion is a function of the distance $Z$ of the snow along the optical axis, the location $(X,Y)$ of the snowflake on the plane tangential to the optical axis, and the car's velocity $V$:
\begin{align}
    U=\frac{V \sqrt{X^2+Y^2}}{Z}.
\end{align}
Thus the dwell time of the snowflake $T$ can be modeled as:
\begin{align}\label{eq:dwell_XY}
    T=\frac{D}{U}=\frac{DZ}{V\sqrt{X^2+Y^2}},
\end{align}
where D is the diameter of the snowflake. Recall that the camera coordinate $(X,Y,Z)$ can be projected onto image plane $(x,y)$ at the focal length $\ell$ by the following relation:
\begin{align}
    (x,y)=\frac{\ell}{Z}(X,Y).
\end{align}
Substituting this into \eqref{eq:dwell_XY} cancels the distance term $Z$ and proves the hypothesis in \eqref{eq:dwell}.

\end{proof}
Some readers may find Lemma 1 surprising---as the vehicle approaches the snowflake (i.e.~$Z$ gets smaller), the apparent \emph{pixel} velocity $u$ of the snowflake increases:
\begin{align}
    u=\frac{\ell}{Z}U.
\end{align}
Indeed, the closer snowflakes appear to move faster in image plane, as evidenced by longer streak in Figure \ref{fig:accumulations}. But the \emph{pixel} diameter of the snow is also inversely proportional to distance $Z$:
\begin{align}
    d=\frac{D}{Z}
\end{align}
Hence as the vehicle approaches the snowflake, the increased pixel speed of the snowflake is cancelled out by the increased pixel diameter of the snowflake:
\begin{align}
    T=\frac{d}{u}=\frac{D/Z}{U/Z}.
\end{align}
Thus we conclude that the dwell time is a function of snowflake diameter $D$ and the velocity $V$. One may also reparameterize dwell time in \eqref{eq:dwell} in terms of the pixel distance $r=\sqrt{x^2+y^2}$ from the optical axis (or more precisely, the direction of the vehicle motion), as follows:
\begin{align}\label{eq:dwell_r}
    T=\frac{D\ell}{V r}.
\end{align}
It is easy to see that the snowflakes at the center of the image plane (small $r$) has longer dwell time than the snowflakes in the periphery (large $r$).

\subsection{Statistical Model of Snowflake Dwell Time}

We rigorously derive the probability density function (pdf) $f_T:\mathbb{R}\to\mathbb{R}$ of the dwell time $T$ for a snowflake with diameter $D$ and distance $Z$ while the car travels at velocity $V$. Although a typical image sensor is rectangle-shaped ($|x|<W/2$ and $|y|<H/2$ where $W$ and $H$ are widths and height, respectively), the math below is considerably simpler if the detector were circular---that is, $r=\sqrt{x^2+y^2}<R$ where $R$ is the detector radius. For this purely mathematical convenience, we derive $f_T(t)$ in terms of {\bf circular detector} first, but subsequently draw conclusions about rectangular detector without the loss of generality. 

We begin with the reasonable assumption that the snow particles are uniformly distributed.  That is, 
\begin{align}
    f_{xy}(x,y)=\begin{cases}
    \frac{1}{R^2\pi}&\text{if $\sqrt{x^2+y^2}<R$}\\
    0&\text{else.}
    \end{cases}
\end{align}
where $R$ is the detector radius. Furthermore, differentiating $P(r\leq r_0)=\frac{r_0^2\pi}{R^2\pi}$ with respect to $r_0$ yields
\begin{align}
    f_r(r)=\begin{cases}
    \frac{2 r}{R^2}&\text{if $r<R$}\\
    0&\text{else.}
    \end{cases}
\end{align}
Treating \eqref{eq:dwell_r} as a function of random variable, pdf of dwell time $T$ take the following form:
\begin{align}\label{eq:dwell_pdf}
    f_T(T|D,V)=&\frac{f_r\left(\frac{D\ell}{VT}\right)}{\left|\frac{\partial}{\partial r}\frac{D\ell}{Vr}\right|}
    = \begin{cases}
    \frac{2D^2\ell^2}{R^2V^2T^3}& \text{if $T>\frac{D\ell}{VR}$}\\
    0&\text{else.}
    \end{cases}
\end{align}
Thus the pdf of the dwell time is proportional to $(D/V)^2$ and inversely proportional to $T^3$. Figure \ref{fig:pdf} shows example $f_T(T|D,V)$ for various vehicle velocities $V$. Most snowflake diameters are said to range between 0.02 inches (0.508mm) to 0.2 inches (5.08mm) \cite{Skilling2006how}. Even for a large diameter $D=5mm$ and moderately slow driving speed (V=20kmh), the dwell time is concentrated below 2ms.
\begin{figure}
\begin{center}
    \includegraphics[width=8cm]{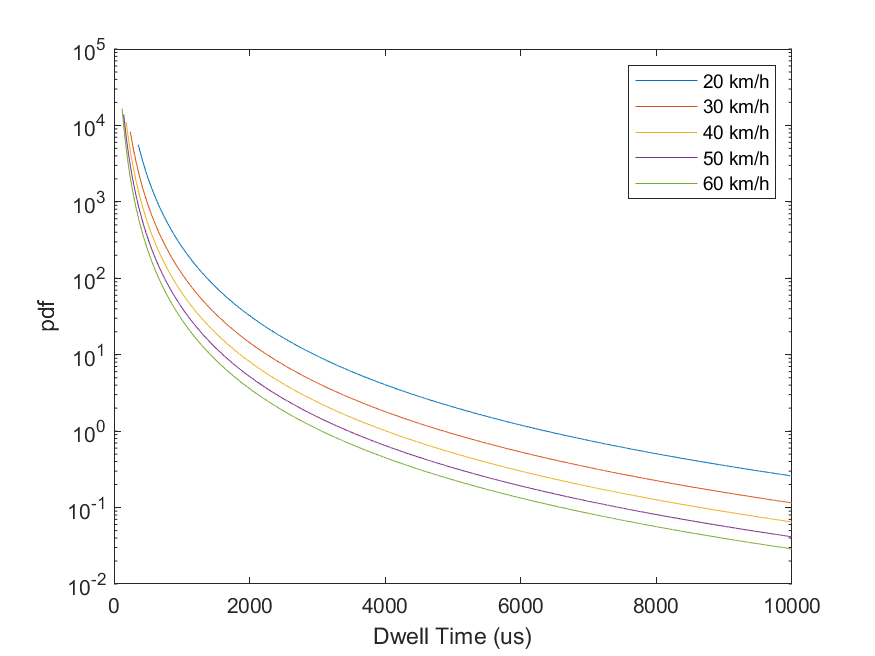}
\end{center}
\caption{Probability density function of the snowflake dwell time $T$, computed assuming a large snowflake diameter of 5mm. As the pdf scales inversely proportional to $T^3$, the distribution is concentrated below 2ms, even at moderately slow driving speeds.}
\label{fig:pdf}
\end{figure}

{\bf For rectangle sensor model}, $P(r\leq r_0)$ is computed by taking an intersection of the rectangle with the $r_0$ radius circle. While not difficult, there are many ``case statements'' to consider (such as when $r_0<H/2$, $H/w<r_0<W$ or $r_0>W$ in landscape mode). Yet, even with the rectangular sensor, the main conclusion that the probability density function $f_T(T|D,V)$ scales proportionally to $(D/V)^2$ and inversely proportional to $T^3$ does not change.

%%%%%%%%%%%%%%%%%%%%%%%%%%%%%%%%%%%%%%%%%%%%%%%%%%%%%%%%%%%%%%%%%%%%%%%%%%%%%%

\section{Proposed: Event-Based Snow Removal} \label{sec:removal-model}

In this section we present the theory and the implementation of the proposed event-based snow removal (EBSnoR) algorithm. At the heart of EBSnoR is the optimal statistical hypothesis testing on snowflake dwell time, which we develop in Section \ref{sec:hypothesis}. We then present a novel technique for measuring the dwell time based on event data in Section \ref{sec:event_dwell_time}. We make practical considerations in Section \ref{sec:discussion}.

\subsection{Optimal Dwell Time Thresholding}
\label{sec:hypothesis}

We formalize snowflake rejection in EBSnoR as a Neyman-Pearson hypothesis testing on the dwell time. That is, we consider the hypotheses:
\begin{align}\label{eq:hypothesis}
    \begin{cases}
    H_0:& \text{snowflake} \\
    H_1:& \text{background.}
    \end{cases}
\end{align}
Define the likelihood ratio function $L(\cdot|V):\mathbb{R}\to\mathbb{R}$ as:
\begin{align}%\label{eq:likelihood_ratio}
    L(T|V)=&\frac{f_T(T|H_1,V)}{f_T(T|H_0,V)},
\end{align}
where $f_T(T|H_i,V)$ is the conditional likelihood function of the hypothesis $H_i$ for a given car speed $V$. Invoking the seminal Neyman-Pearson Lemma \cite{NP}, thresholding performed on $L(T|V)$ is a provably optimal binary classification of the hypothesis. That is, the \emph{decision rule} $\delta$ takes the form
\begin{align}\label{eq:LRT}
    \delta(T)=\begin{cases}
    1&\text{if $L(T|V)\geq \kappa$}\\
    0&\text{if $L(T|V)<\kappa$}
    \end{cases}
\end{align}
for some threshold value $\kappa$. Working with likelihood ratio test is challenging because the likelihood functions $f_T(T|H_i,V)$ are not explicitly defined, however. Nevertheless, strong theoretical results such as {\bf Theorem \ref{thm:monotonic}} below can be proven.

\begin{theorem}\label{thm:monotonic}
Let $\theta$ be the maximum physically realizable snowflake size in nature. Then the likelihood ratio function $L(T|V)$ is a monotonically decreasing function for $T>\frac{\theta\ell}{VR}$.
\end{theorem}
\begin{proof}
Let $f_D(D|H_0)$ denote pdf of snowflake diameters $D$. We assume $f_D(D|H_0)=0,\forall D>\theta$ because $D$ cannot exceed $\theta$. Hence the null likelihood function takes the following form:
\begin{align}\label{eq:likelihood_H0}
\begin{split}
    f_T(T|H_0,V)=&\int_0^{\infty} f_T(T|D,V)f_D(D|H_0)dD\\
    =&\int_0^{\theta} \frac{2D^2\ell^2}{R^2V^2T^3}\phi\left(T-\frac{D\ell}{VR}\right)f_D(D|H_0)dD\\
    =&\int_0^{\min(\theta,\frac{TVR}{\ell})} \frac{2D^2\ell^2}{R^2V^2T^3}f_D(D|H_0)dD,
\end{split}
\end{align}
where $\phi:\mathbb{R}\to\mathbb{R}$ is the unit step function. 

Similarly, let $f_D(D|H_1)$ denote pdf not of snowflake diameters, but of physical dimensionality of \emph{any other} details in background scenery of interest. When the background motion is negligible relative to the vehicle's velocity (a point we will revisit in Section \ref{sec:discussion}), the dwell time pdf in \eqref{eq:dwell_pdf} applies to background object details as well. Hence the likelihood function of the alternative hypothesis takes the following form:
\begin{align}\label{eq:likelihood_H1}
    f_T(T|H_1,V)
        =&\int_{0}^{\frac{TVR}{\ell}} \frac{2D^2\ell^2}{R^2V^2T^3}f_D(D|H_1)dD.
\end{align}

Combining \eqref{eq:likelihood_H0} and \eqref{eq:likelihood_H1}, we arrive at the likelihood ratio for $T>\frac{\theta\ell}{VR}$
\begin{align}\label{eq:likelihood_ratio}
\begin{split}
    L(T|V)=&\frac{f_T(T|H_1,V)}{f_T(T|H_0,V)}\\
    =&
    \frac{\int_{0}^{\frac{TVR}{\ell}} \frac{2D^2\ell^2}{R^2V^2T^3}f_D(D|H_1)dD}{\int_0^{\theta} \frac{2D^2\ell^2}{R^2V^2T^3}f_D(D|H_0)dD}\\
    =&
    \frac{\int_{0}^{\frac{TVR}{\ell}}  D^2f_D(D|H_1)dD}{\int_0^{\theta} D^2f_D(D|H_0)dD}.
\end{split}
\end{align}
The above likelihood ratio is a monotonically increasing function of dwell time $T$ since $D^2$ and $f_D$ are non-negative values.
\end{proof}
It is important to emphasize that the monotonicity of the likelihood ratio in {\bf Theorem \ref{thm:monotonic}} was rigorously proven without explicitly defining or knowing the snowflake diameter pdf $f_D(D|H_0)$ nor the background detail pdf $f_D(D|H_1)$. As stated earlier, the maximum flake size is reported to be around $\theta=5$mm \cite{Skilling2006how}.

The major significance of {\bf Theorem \ref{thm:monotonic}} is that the optimal thresholding is the thresholding of $T$. That is, the optimal binary classifier in \eqref{eq:LRT} is equivalent to a simple thresholding of the dwell time:
\begin{align}\label{eq:LRT_t}
    \delta(T)=\begin{cases}
    1&\text{if $T\geq \eta$}\\
    0&\text{if $T<\eta$.}
    \end{cases}
\end{align}
Thus the proposed EBSnoR carries out snowflake rejection using the simple dwell time thresholding in \eqref{eq:LRT_t}.

Next, we derive the false positive rate (snowflake miss-classified as background) and the false negative rate (background miss-classified as snowflake) of the hypothesis test in \eqref{eq:hypothesis} in {\bf Corollary \ref{cor:false_alarm}}.
\begin{corollary}\label{cor:false_alarm}
Suppose we set the threshold $\eta$ as proportional to the critical dwell time of the largest snowflake at velocity $V$:
\begin{align}\label{eq:tau}
\eta=\tau\frac{\theta}{V}.
\end{align}
Then the false positive rate and the false negative rate of the likelihood ratio test in \eqref{eq:LRT_t} is independent of the car speed $V$.
\end{corollary}
\begin{proof}
The false positive rate  in terms of the threshold in \eqref{eq:tau} can be derived as follows:
\begin{align}\label{eq:PF}
\begin{split}
    P\left(T>\eta|H_0,V\right)&=\int_0^{\theta}\int_{\eta}^{\infty}\frac{2D^2\ell^2}{R^2V^2T^3}f_D(D|H_0)dTdD\\
%    &=-\int_0^{\theta}\left.\frac{D^2\ell^2}{R^2V^2T^2}\right|_{\tau}^{\infty}f_D(D|H_0)dD\\
    &=\int_0^{\theta}\frac{D^2\ell^2}{R^2\theta^2\tau^2}f_D(D|H_0)dD.
\end{split}
\end{align}
The false negative rate is slightly more complicated, as shown below:
\begin{align}\label{eq:PM}
\begin{split}
    &P\left(T\leq \eta|H_1,V\right)
    =\int_0^{\frac{\eta VR}{\ell}}\int_{\frac{D\ell}{VR}}^{\eta}\frac{2D^2\ell^2}{R^2V^2T^3}f_D(D|H_1)dTdD\\
%    &=\int_0^{\frac{\eta VR}{\ell}}\int_{\eta}^{\infty}\frac{2D^2\ell^2}{R^2V^2T^3}f_D(D|H_1)dTdD\\
%    &~+\int_{\frac{\eta VR}{\ell}}^{\infty}\int_{\frac{D\ell}{VR}}^{\infty}\frac{2D^2\ell^2}{R^2V^2T^3}f_D(D|H_1)dTdD\\
    &=\int_0^{\frac{\eta \theta R}{\ell}}\left(1-\frac{D^2\ell^2}{R^2\theta^2\tau^2}\right)f_D(D|H_1)dD.%+\int_{\frac{\tau \theta R}{\ell}}^{\infty}1\cdot f_D(D|H_1)dTdD.
\end{split}
\end{align}
We conclude from \eqref{eq:PF} and \eqref{eq:PM} that setting the threshold $\eta$ to be inversely proportional to the vehicle velocity $V$ makes the false positive and false negative rates independent of $V$.
\end{proof}

\begin{figure}
    \begin{center}
    \begin{tabular}{cc}
        \includegraphics[width=4cm]{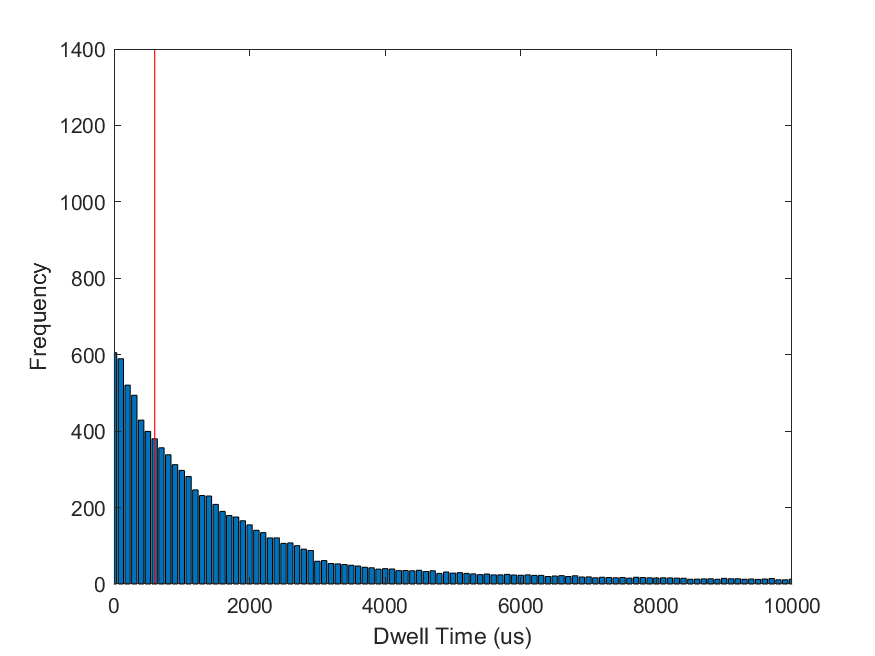} &
        \includegraphics[width=4cm]{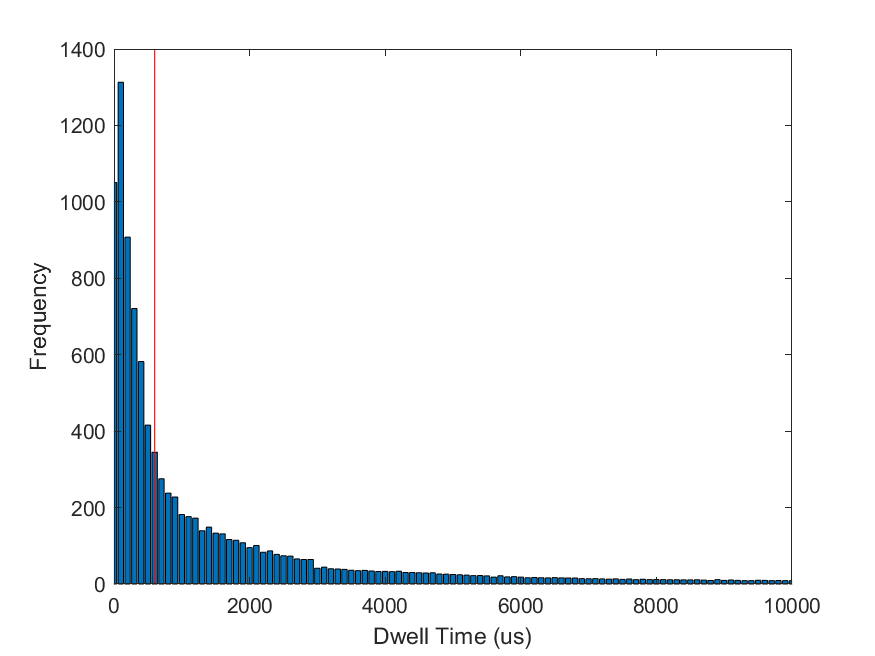} \\
        (a) Background only & (b) Snow + Background
    \end{tabular}
    \end{center}
    \caption{Histograms of measured dwell times for (a) overcast day and (b) snowy day.}
    \label{fig:hists}
\end{figure}

One fortunate outcome of {\bf Corollary \ref{cor:false_alarm}} is that a single threshold parameter $\tau$ would yield a consistent performance across any car speed $V$ (i.e.~the threshold $\eta$ adjusts for velocity based on $\tau$). Neyman-Pearson hypothesis testing comes in two flavors. {\bf The first version} sets the threshold $\tau$ to satisfy the false positive rate at the significance level $\alpha$:
\begin{align}\label{eq:eta1}
    P\left(T>\eta|H_0,V\right)=P\left(\left.T>\tau\frac{\theta}{V}\right|H_0,V\right)<\alpha.
\end{align}
The Neyman-Pearson Lemma provides a theoretical guarantee that the hypothesis test in \eqref{eq:LRT_t} with threshold $\tau$ satisfying \eqref{eq:eta1} maximizes the detection at this significance level \cite{NP}. The false positive rate itself is tightly upper-bounded:
\begin{align}
\begin{split}
    P\left(T>\eta|H_0,V\right)%=&\int_0^{\theta}\int_{\eta}^{\infty}\frac{2D^2\ell^2}{R^2V^2T^3}f_D(D|H_0)dTdD\\
%    &=-\int_0^{\theta}\left.\frac{D^2\ell^2}{R^2V^2T^2}\right|_{\tau}^{\infty}f_D(D|H_0)dD\\
%    &=\int_0^{\theta}\frac{D^2\ell^2}{R^2\theta^2\tau^2}f_D(D|H_0)dD\\
    &<\frac{\ell^2}{R^2\tau^2}.
\end{split}
\end{align}
(proof: substitute $\theta$ for $D^2$ in \eqref{eq:PF}.) Thus in this work, we choose the threshold $\tau$ according to the following rule:
\begin{align}\label{eq:NPalpha}
\tau=\frac{\ell}{R\sqrt{\alpha}}.
\end{align}

{\bf The second version} of Neyman-Pearson hypothesis testing sets the threshold $\tau$ to satisfy the false negative rate at the significance level $\beta$:
\begin{align}\label{eq:miss}
     P\left(T<\eta|H_1,V\right)=P\left(\left.T<\tau\frac{\theta}{V}\right|H_1,V\right)<\beta.
\end{align}
The threshold determined in this manner guarantees that hypothesis test in \eqref{eq:LRT_t} maximizes snowflake rejection at this significance level. Unlike the earlier version of Neyman-Pearson hypothesis, this version must set the threshold $\tau$ empirically because $f_D(D|H_1)$ cannot be defined explicitly, however. In this work, we handled this by driving a car \emph{on an overcast day with no precipitation} at a given speed $V_0$ while collecting data with the event camera. With no snowflakes in the scene, the empirical distribution of the dwell time (computed from the event data using the technique described in Section \ref{sec:event_dwell_time}) is a baseline proxy for the likelihood function $f_T(T|H_1,V_0)$ at speed $V_0$. See example in Figure \ref{fig:hists}(a). Thanks to {\bf Corollary \ref{cor:false_alarm}}, setting the threshold $\tau$ by exhaustive search to empirically satisfy 
\begin{align}\label{eq:eta2}
    P\left(\left.T<\tau\frac{\theta}{V_0}\right|H_1,V_0\right)=\beta
\end{align}
will determine the value of $\tau$ that generalizes to other speeds $V\neq V_0$. By Neyman-Pearson lemma, threshold $\tau$ satisfying \eqref{eq:eta2} guarantees maximum snow detection at the significance level $\beta$\cite{NP}. For comparison, Figure \ref{fig:hists}(b) shows the dwell time histogram of snowy day event sequence, although this histogram represents a combination of snowflake and background dwell times. Nevertheless, there is a stark contrast between the two histograms, where Figure \ref{fig:hists}(b) exhibits the characteristics seen in Figure \ref{fig:pdf}.

\subsection{Event-Based Dwell Time Measurement}
\label{sec:event_dwell_time}

\begin{figure*}
    \centering
    \begin{tabular}{@{}c@{}}
    \includegraphics[width=.45\textwidth]{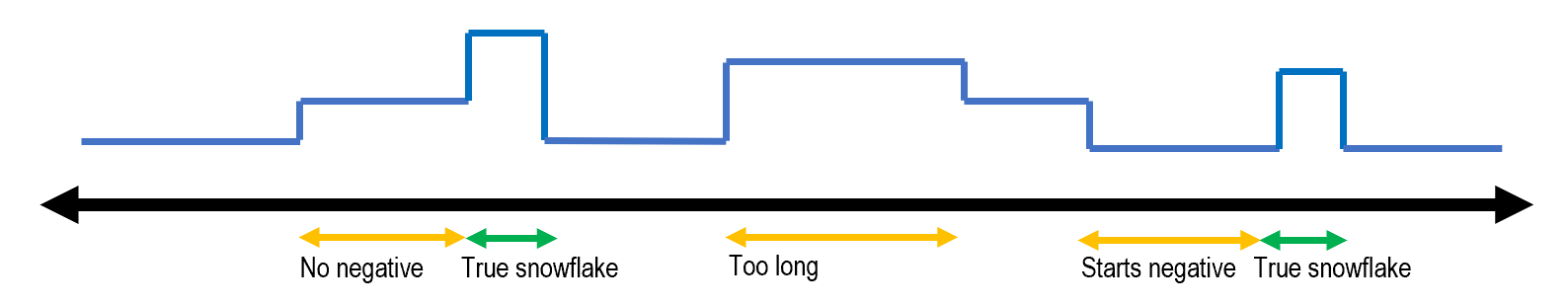}\\
    (a) Log-intensity as a function of time\\
    \includegraphics[width=.45\textwidth]{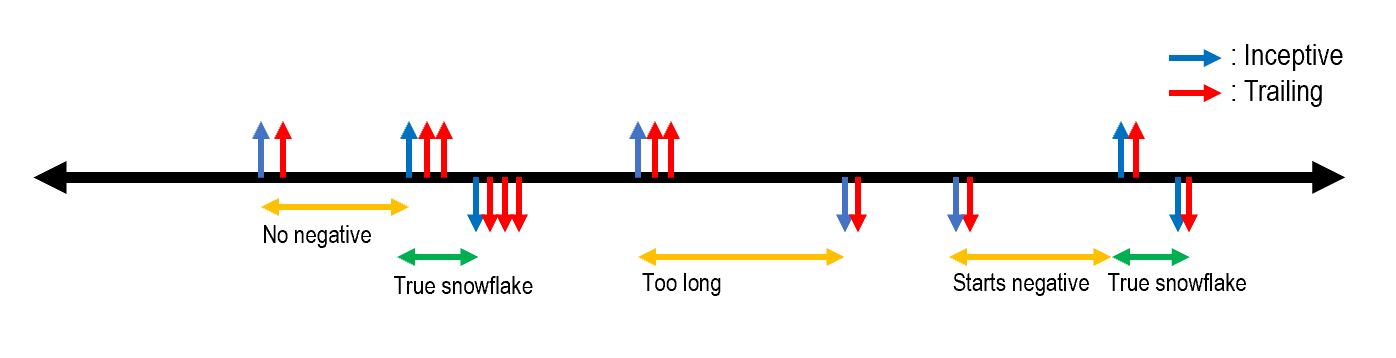}\\ 
    (b) Events as a function of time
    \end{tabular}
\begin{tabular}{c}
\includegraphics[width=.3\textwidth]{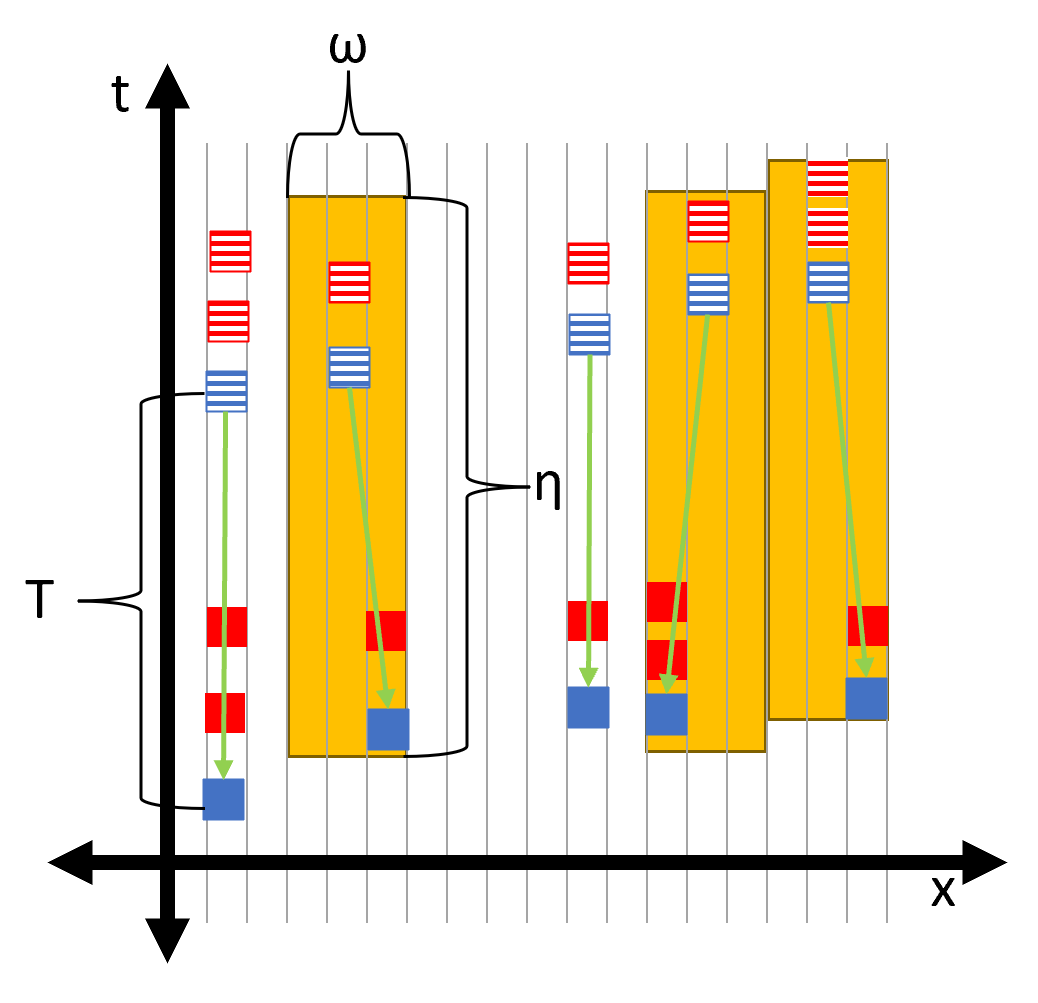}\\
     (c) Spatial-temporal windowing on events
\end{tabular}
    \caption{Pixel timing diagram for (a) log-intensity and (b) events. Dwell time can be computed as the timestamp difference between positive and negative inceptive events.
    (c) Spatial-temporal window used by EBSnoR to threshold dwell time $T$ by the parameter $\eta$. Due to the potential for missing events, we look within the spatial window $\omega$ to find the positive-negative edge pairs.}
    \label{fig:timing}
\end{figure*}

The proposed EBSnoR is an optimal binary classification based on the decision rule in \eqref{eq:LRT_t}, carried out by thresholding the dwell time $T$. In this section, we develop a technique to measure the dwell time $T$ from event data. We begin with the underlying assumption that the snow is brighter in relation to the surrounding and background intensity values. A snowflake appears luminous and white because it is made up of tiny translucent ice crystals that transmit or reflect light with minimal absorption\cite{Koenderink:92}. The contrast with background is further exaggerated by the dimmer environmental light available in overcast sky during typical snowfall.

Consider the intensity timing plot in Figure \ref{fig:timing}(a). Due to the relative brightness of the snowflake compared to the surrounding background radiance, we see a positive step (positive edge) in the intensity at the moment that a snowflake intersects a pixel's field of view. That intensity will remain high during the snowflake's dwell time, and the negative intensity step (negative edge) is seen when the snowflake exists this particular pixel.

Thus the dwell time can be inferred from the intensity timing plot by measuring the time between a positive edge followed by a negative edge. The other edge intervals---``positive edge to positive edge,'' ``negative edge to positive edge,'' and ''negative edge to negative edge''---can be discarded since they are not consistent with the snowflake model above, and therefore can be attributed to the background. In practice, however, measuring the millisecond-order snowflake dwell time in this manner using conventional framing cameras is difficult because the framerate is too slow.

Using the above principles, the proposed EBSnoR measures the snowflake dwell time from event streams in several stages. The system-level block diagram of EBSnoR is shown in Figure \ref{fig:algo-function}. We explain the method in detail by comparing the intensity timing plot in Figure \ref{fig:timing}(a) to the corresponding event timing plot in Figure \ref{fig:timing}(b). The microsecond-order resolution of the event timing makes the event-based sensors more attractive for dwell time measurement. Recalling Section \ref{sec:IEF}, the arrival of snowflake at a particular pixel is marked by the timestamp of the positive inceptive event. Likewise, the timestamp of the subsequent negative inceptive event coincides with the exact moment of the snowflake exit. Thus, the dwell time is measured between the timestamps of the positive inceptive event and its following negative inceptive event at the same pixel. See Figure \ref{fig:timing}(b) for the timing of inceptive events (green arrows) corresponding to snowflakes. 

\begin{figure*}
\begin{center}
    \begin{tabular}{cc}
        \includegraphics[width=.5\textwidth]{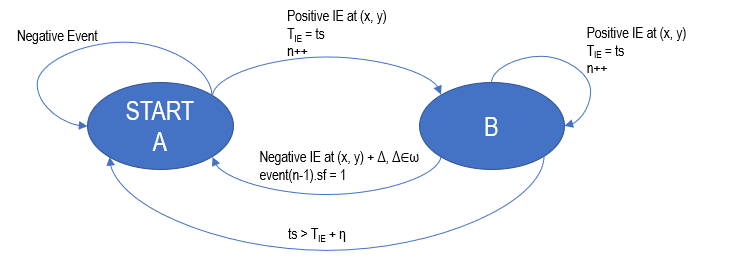} &
        \includegraphics[width=.5\textwidth]{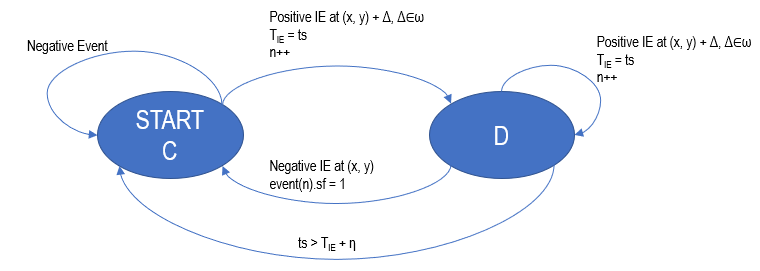} \\
        (a) Processing of positive inceptive event&
        (b) Processing of negative inceptive event
    \end{tabular}
\end{center}
    \caption{State diagram for detecting snow from inceptive events. When an inceptive event is declared as snowflake (\texttt{event(n).sf=1}), we assume that the corresponding trailing events are also labeled as snowflakes.}
    \label{fig:state-EBSnoR}
\end{figure*}

Taking real-world event-based sensor hardware into consideration, the dwell time measurement algorithm must account for ``missed events.'' That is, sometimes the pixels are not sensitive enough to generate events, even when a pixel encounters a large log-intensity edge. For this reason, we search over neighboring pixels to find positive and negative inceptive events in spatial-temporal proximity (temporal threshold of $\eta$ and spatial window $\omega$). This is illustrated in Figure \ref{fig:timing}(c), where spatially neighboring inceptive events allows us to recover from missing events to make dwell time measurements.

The algorithm is summarized in the state transition diagrams in Figure \ref{fig:state-EBSnoR}. For efficient implementation of this algorithm, we process the event based snowflake dwell time measurement as a FIFO (first in, first out) structure arranged in pixels. Referring to the positive event processing in Figure \ref{fig:state-EBSnoR}(a), when a positive polarity inceptive event at pixel location $(x,y)$ is encountered, the event is stored in the FIFO at the corresponding location, and the state of the pixels move to ``B.'' If the negative inceptive event is encountered within spatial window $(x,y)\pm\omega$ before $\eta=\tau \theta/V$, then snowflake is detected. The previous positive inceptive event (as well as the associated trailing events) are marked as snowflake (\texttt{event(n-1).sf=1}),
and the state returns to ``A.''  Similarly, the negative polarity inceptive events are processed according to the state diagram in Figure \ref{fig:state-EBSnoR}(b). When a positive polarity inceptive event occurs within the spatial neighborhood $(x,y)\pm\omega$, its timing is recorded and the state moves to ``D.''  If the negative inceptive event is encountered at pixel $(x,y)$ before $\eta=\tau \theta/V$, then snowflake is detected. This new negative inceptive event and the corresponding trailing events are marked as snowflake (\texttt{event(n).sf=1}) as the state returns to ``C.''

Unlike intensity cameras, removing unwanted data from event camera footage is as simple as removing undesired events from the data stream. Thus, once a list of snow events has been obtained, they are deleted from this list of event data points.

%%%%%%%%%%%%%%%%%%%%%%%%%%%

\subsection{Discussions}
\label{sec:discussion}
Let us address several practical considerations for detecting snowfakes by thresholding dwell time $T$. First, camera parameters such as focal length and detector size $R$ are fixed and known numbers. The maximum snowflake diameter $\theta$=5mm is reasonable \cite{Skilling2006how}. In automotive applications, vehicle velocity $V$ can be provided by the odometer (although visual odometery techniques can also replace traditional odometery).

In the baseline hypothesis testing in \eqref{eq:likelihood_H1}, we made an assumption that the background object velocity is insignificant compared to the vehicle motion $V$. Generalizing to the case that the object velocity is insignificant (e.g.~background object is a car) replaces the denominator $V$ with apparent motion relative to the vehicle/camera $V'$. This does not change the main conclusion that the likelihood ratio in \eqref{eq:likelihood_ratio} is monotonic; and Neyman-Pearson threshold minimizing false positive rate is determined independent of $V'$. 

The monotonicity of likelihood ratio in {\bf Theorem \ref{thm:monotonic}} is only guaranteed for $T>\frac{\theta\ell}{VR}$. To understand why this has negligible impact in practice, substitute Neyman-Pearson threshold in \eqref{eq:NPalpha} into \eqref{eq:tau}:
\begin{align}
    \eta=\frac{\theta\ell}{VR\sqrt{\alpha}}.
\end{align}
Since the significance value $\alpha$ is less than 1 (and usually very small), we conclude that the practical operating range of the threshold $\eta$ is far larger than the critical boundary $\frac{\theta\ell}{VR}$.

Revisiting \eqref{eq:PM}, consider the false negative rate of the background object with a specific detail size $D_0$ at the Neyman-Pearson significance level of $\alpha$, :
\begin{align}\label{eq:miss2}
\begin{split}
    &P(T\leq \eta|H_1,D_0,V)=\int_{\frac{D_0\ell}{VR}}^{\eta}\frac{2D_0^2\ell^2}{R^2V^2T^3}dT\\
        &=\max\left(1-\frac{D_0^2\ell^2}{R^2\theta^2\tau^2},0\right)\\
        &=\max\left(1-\alpha\frac{D_0^2}{\theta^2},0\right).
\end{split}
\end{align}
Since this is a monotonically decreasing function of $D_0$ (at any distance $Z$, thanks to Lemma \ref{lem:dwell_time}), the background details falsely rejected by EBSnoR affects small spatial details. For instance, at $\alpha$=0.05 significance rate, background details of $D_0$=1.58cm will yield 50\% false negative rate; false negative rate is 0\% for details larger than $D_0$=2.23cm. Potential loss of spatial details with physical dimensions smaller than 2cm (again, at any distance $Z$) is believed to result in negligible performance loss for computer vision algorithms designed to detect large objects such as vehicles, street signs, and pedestrians.

\begin{figure}
    \begin{center}
        \includegraphics[width=7cm]{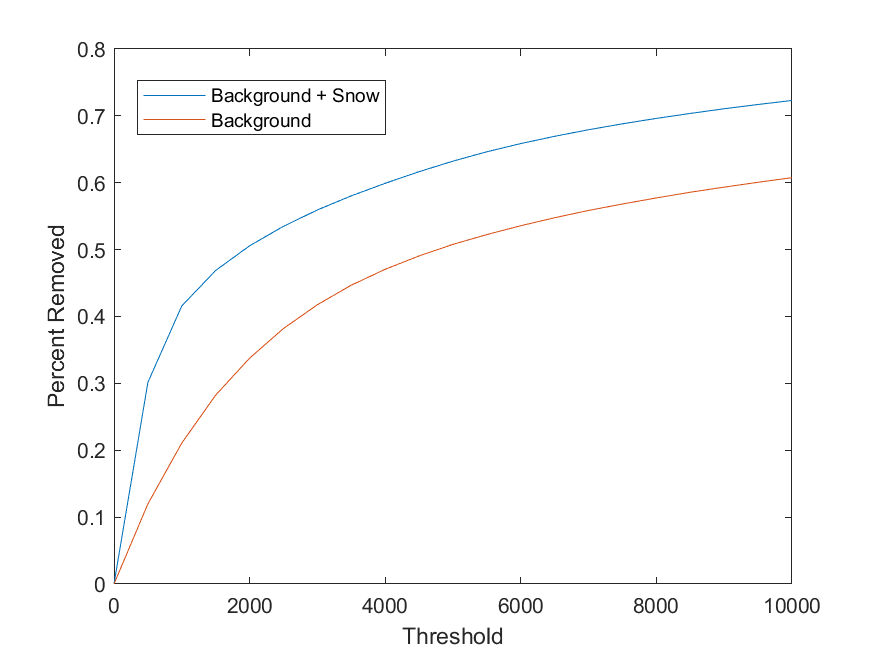}
    \end{center}
    \caption{Percentage of events removed by EBSnoR, as a function of the EBSnoR threshold. For snowy day sequence (``background + snow''), very high percentage of events fall below the dwell time threshold. By contrast, the baseline sequence (``background'') is far less sensitive to small Neyman-Pearson threshold values.}
    \label{fig:percent}
\end{figure}

Lastly, recall that the event-based dwell time measurement described in Section \ref{sec:event_dwell_time} measures the time between consecutive positive and negative incceptive events. Thus the false negative rates in \eqref{eq:miss} and \eqref{eq:miss2} are \emph{exaggerated}, in the sense that threhsolding does not impact events corresponding to consecutive positive-positive, negative-positive, and negative-negative edges. This fact is reflected in Figure \ref{fig:percent}, where we plot the percentage of events removed by EBSnoR as a function of threshold $\eta$. It can be seen that the slope of the removed pixels is very steep for snow sequence compared to the ``background only'' events near $\eta=0$. This implies that the proposed thresholding effectively discriminates snowflakes from the background.

%%%%%%%%%%%%%%%%%%%%%%%%%%%%%%%%%%%%%%%%%%%%%%%%%%%%%%%%%%%%%%%%%%%%%%%%%%%%%%

\section{Experiments} \label{sec:experiments}

\subsection{UDayton22EBSnow Dataset} \label{sec:dataset}

\begin{figure*}[ht]
\begin{center}
    \begin{tabular}{ccc}
         \includegraphics[width=5cm]{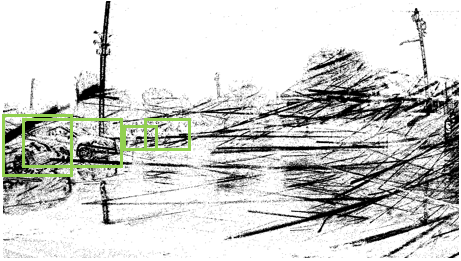} &
         \includegraphics[width=5cm]{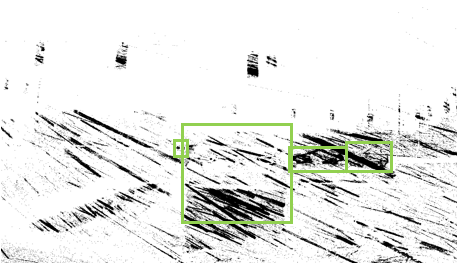} &
         \includegraphics[width=5cm]{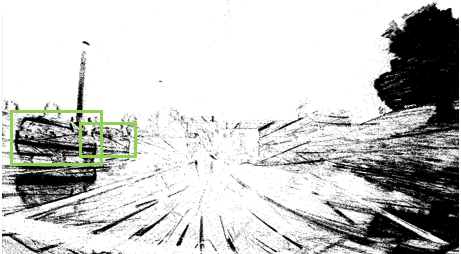} \\
         (a) City & (b) Suburb & (c) Highway
    \end{tabular}
\end{center}
\caption{Example of UDayton22EBSnow dataset events and annotated bounding box around vehicles.}
\label{fig:dataset-ex}
\end{figure*}

\begin{table*}
    \centering
    \caption{Key statistics of annotated event sequences in the UDayton22EBSnow dataset.}
    \label{tab:sequences}
    \begin{tabular}{|c||c|c|c|c|c|}
\hline
        \multirow{2}{*}{Sequence Name} &\multirow{2}{*}{\# Events} &\multicolumn{2}{c|}{Total Sequence} & \multicolumn{2}{c|}{Stopped at Traffic}\\
        &&Duration&\# Bounding Boxes&Duration &\# Bounding Boxes\\
\hline\hline
City &   1.2 billion & 155s & 725 & 18s & 20\\
Suburb &  6.5 billion & 409s & 362 & 36s & 50\\
Highway &  3.3 billion & 176s & 109 & 0s & 0\\
\hline
\end{tabular}
\end{table*}

For testing the effectiveness of snowflake removal, we collected a new dataset we call UDayton22EBSnow comprised of 9 sequences of driving through snowfall using Prophesee EVK2 HD event-based camera, with 5mm focal-length lens yielding horizontal and vertical fields of view of 63.76 degrees and 38.58 degrees, respectively. The camera was mounted on the dashboard (slightly off-centered towards the passenger side) inside the vehicle front-facing on a moving vehicle navigating through highway, city, and suburb roads. As expected, the highway driving resulted in fast apparent motion of the snowflakes with very short dwell times. City driving sequence at times involve slow ego-motion of the vehicle (recall that the theoretical development in Section \ref{sec:dwell-time-model} assumed that the snowflake speed was negligibly small compared to the vehicle speed). The suburb roads resulted in mostly steady vehicle ego-movement. Captured sequences range between two and five minutes in length. Some sequences contain scenes where the car was stopped at the traffic light. During the stop, the generated events were almost entirely due to snowflake movements and the surrounding vehicles that are still moving.

In three of the sequences, we manually annotated bounding boxes around vehicles---one event sequence from each of highway, city, and suburb drives. Annotations were made at \emph{one second intervals}, and we also recorded the timestamps of the windshield wiper that sweeps across the camera's entire field of view. Examples of the labeled cars are shown in Figure \ref{fig:dataset-ex}. The statistics of the event streams are summarized in Table \ref{tab:sequences}.  We caution that the number of annotated bounding boxes does not necessarily represent the number of unique cars---a car staying in the field of the camera's view for more than one second will be annotated multiple times. A portion of the ``City'' sequence is spent driving through a parking lot, where there is a high number of cars that were annotated. \emph{This dataset will be made available to the public upon acceptance of this paper.}

\subsection{Event-Based Snow Removal Results}

\begin{figure} [ht]
   \begin{center}
   \includegraphics[width=9cm]{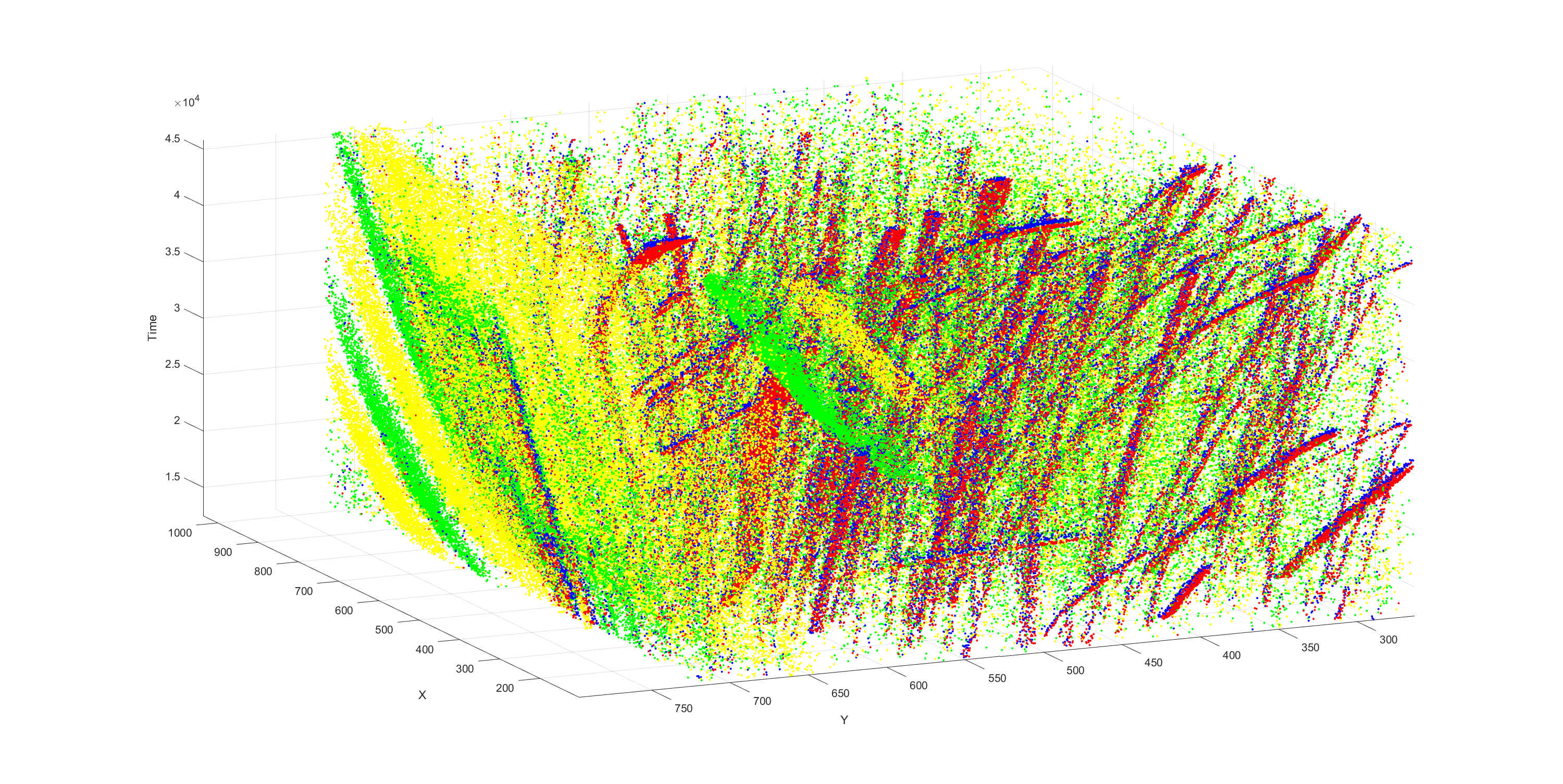}
   \end{center}
   \caption{EBSnoR partitions event stream into snowflakes (red=positive, blue=negative) and background (yellow=positive, green=negative) events. Applied to UDayton22EBSnow ``Highway'' sequence. It can be seen that snowflake form a track of positive event, followed immediately by a negative event track---the space between them represents the snowflake dwell time.}
   \label{fig:snow-polarity}
\end{figure}

\begin{figure*}
\begin{tabular}{cc}
\animategraphics[autoplay, loop, width=9cm]{12}{figures/suburb_gif/frame_}{1}{100}&
\animategraphics[autoplay, loop, width=9cm]{12}{figures/highway_gif/frame_}{1}{100}\\
(a) Highway & (b) Suburb
\end{tabular}
\caption{EBSnoR partitions event stream into snowflakes (red) and background (green=positive, blue=negative) events. Applied to UDayton22EBSnow ``Highway'' and ``Suburb'' sequences and played back at 0.3$\times$ speed.}
\label{fig:results}
\end{figure*}

Recall Corollary \ref{cor:false_alarm} where the Neyman-Pearson threshold $\eta$ scales inversely proportional to ego-motion velocity $V$. Although this formulation has a nice guarantees in terms of false positive and false negative rates, we did not have a technical capability to record vehicle velocity when UDAYTON22EBSnoR data was being collected. For this reason, we simply fixed the threshold to $\eta=3000ms$. This may be suboptimal at very fast and very slow ego-motion velocities, but the results were acceptable in most sequences.

Figure \ref{fig:snow-polarity} shows the result of applying EBSnoR to snow sequence. As evidenced by the detected snowflake tracks, snowflakes are brighter than the surrounding environment---the arrival and the departure of snowflake at a pixel are marked by positive (red) and negative (blue) events, respectively. In the time-space representation like the one in Figure \ref{fig:snow-polarity}, these positive and negative event form parallel tracks, where the time interval between them represents the snowflake's dwell time. 

By contrast, typical background events form two dimensional manifolds stemming from the edges that travel across time. As evidenced by Figure \ref{fig:snow-polarity}, the events forming these manifolds often share the same polarity (at least within some local regions), and has more gentle slopes when compared to snowflakes because of the slower apparent motion on camera's image plane. Though there are exceptions, by-and-large the we do not commonly see parallel positive and negative manifolds. Even if they do, there is a larger time gaps separating them, compared to typical snowflake dwell time. Thresholding by event-based dwell time therefore safely preserves the events corresponding to the background. 

EBSnoR applied to event sequences in UDayton22EBSnow are shown in Figures \ref{fig:teaser} and \ref{fig:results}. By visual inspection, we can confirm that most snowflake events are correctly identified by the Neyman-Pearson thresholding. The exceptions (false positive) were very large snowflakes with low apparent speed found near the center of ego-motion, resulting in longest dwell time. The snowflakes did not have enough contrast against the overcast sky to generate events.

The background events are largely preserved as well. Road markings, cars, and structures in the environment are correctly classified as background. In the highway sequence, there seems to be instances where the thin branches are miss-classified as snowflake because of low dwell time. Although it is difficult to determine the level of false negatives (background classified as snowflake) in the textured regions such as leaves in ``Suburb'' sequence, events classified as snowflake in the leaves region generally form streaks consistent with the behavior of snowflakes.

\subsection{Car Detection Results}

Because manual annotation of events corresponding to snow is an impractical task, quantitatively evaluating the performance of the event-based snow removal algorithm on this dataset is difficult. Instead, we propose to assess the effectiveness of EBSnoR indirectly by applying an \emph{event-based car detection} algorithm to snow-removed event stream as a proxy. Specifically, we used convolutional neural network (CNN) based car detection algorithm provided in the Prophesee Metavision Toolkit  \cite{de2020large} that was trained on scenes with no snow. We used the network with no modifications---default settings and thresholds, and with no retraining. The network outputs bounding boxes, which we compare to the manually annotated bounding boxes (ground truth).

We use common evaluation metrics for detection tasks, such as recall, precision, the average percent overlap (PO) and the average intersection over union (IOU) of bounding boxes. For scoring, we ignored the ground truth and the detected bounding boxes within 0.3ms of the wiper blade sweep. For benchmarking, we repeat the car detection experiment on the entire unprocessed (i.e.~not snow-removed) event stream---we refer to this as the ``baseline.''

The results are reported in Table \ref{tab:numerical_res}. As one can see, all metrics suggest that the car detection task improves when pre-processed by EBSnoR compared to the unprocessed event streams. The default threshold on the Prophesee Metavision Toolkit's CNN seemed to be tuned to yield high precision, erring on the side of lower recall value. Nevertheless, the largest gain was in the \emph{recall} in the ``City'' sequence, where number of cars detected was considerably higher with the snow-removed event stream while maintaining precision above 88\%. Recalling that the number of bounding boxes in Table \ref{tab:sequences} does not represent unique cars, most of the misses (opposite of recall) by both EBSnoR and baseline results occurred when the cars first appeared into the scene. Most cars were eventually detected, but EBSnoR-preprocessed car detection tended to recognize the car earlier than the ``not snow-removed'' event stream---leading to a higher recall percentages. In time-critical applications such as automotive imaging, early detection of cars with EBSnoR-processing is a distinct advantage.

\begin{table*}\label{tab:numerical_res}
\caption{Car detection results using annotated UDayton22EBSnow event sequences. Snow removed events yield higher average percent overlap (PO), average intersection over union (IOU), precision, and recall when compared to the baseline event stream containing both snow and background events. The CNN detection method in \cite{de2020large} was used without modification (i.e.~trained without snow in consideration).}
\begin{tabular}{|c||c|c|c|c||c|c|c|c|c|}
\hline
&\multicolumn{4}{c||}{EBSnoR (snow removed)}&\multicolumn{4}{c|}{Baseline (not snow removed)}\\
\hline
    Sequence & Avg PO & Avg IOU & Precision & Recall & Avg PO & Avg IOU & Precision & Recall \\
    \hline\hline
    City & 41.822\% & 28.508\% & 88.265\% (346/392) & 47.241\% (346/725) &  41.317\% & 27.200\% & 87.768\% (287/327) & 39.586\% (287/725)  \\
    Suburb & 48.148\% & 37.963\% & 77.358\% (41/53) & 11.325\% (41/362) & 44.716\% & 33.148\% & 60.345\% (35/58) & 9.668\% (35/362)\\
    Highway & 46.876\% & 36.804\% & 81.081\% (30/37) & 27.523\% (30/109) &   42.097\% & 33.282\% & 80.555\% (29/36) & 26.606\% (29/109)\\
    \hline
\end{tabular}
\end{table*}

% \begin{table*}\label{tab:numerical_res}
% \caption{with snow}
% \begin{tabular}{|c||c|c|c|c|c||c|c|c|c|c|c|}
% \hline
% &\multicolumn{5}{c||}{EBSnoR (snow removed)}&\multicolumn{5}{c|}{Baseline (not snow removed)}\\
% \hline
%     Sequence & Avg PO & Avg IOU & Precision & Recall & FP & Avg PO & AVG IOU & Precision & Recall & FP\\
%     \hline\hline
%     City & 41.822\% & 28.508\% & 0.8826 & 0.4772 & 46 &  41.317\% & 27.200\% & 0.8790 & 0.3958 & 40 \\
%     Highway & 46.876\% & 36.804\% & 0.8219 & 0.2752 & 7 &  42.097\% & 33.282\% & 0.8169 & 0.266 & 7\\
%     Suburb & 48.148\% & 37.963\% & 0.7809 & 0.1132 & 12 & 44.716\% & 33.148\% & 0.6034 & 0.0966 & 23\\
%     \hline
% \end{tabular}
% \end{table*}

%%%%%%%%%%%%%%%%%%%%%%%%%%%%%%%%%%%%%%%%%%%%%%%%%%%%%%%%%%%%%%%%%%%%%%%%%%%%%%%%%%%%%%%%%%%%%%%%%%%%%%%%%%%%

\section{Conclusion}
In this paper, we proposed a novel event-based snowflake removal algorithm called EBSnoR. Based on the rigorously derived probability density function of the snowflake dwell time and the monotonicity proof of its likelihood ratio, we developed a Neyman-Pearson hypothesis test to detect snowflake by thresholding dwell time. We also develop a method to measure the dwell time from the event stream, based on which we perform the hypothesis test to partition the event stream into snowflakes and background events. The performance of the proposed EBSnoR is assessed using UDayton22EBSnow dataset comprised of city, subsurb, and highway drives through snow. By visual inspection, we verified that the algorithm does an excellent job at detecting snow. Although textures (such as leaves) were sometimes miscategorized as snowflakes, EBSnoR by and large did an excellent job at detecting the snow. Quantitative evaluation was done by performing event-based car detection algorithm on EBSnoR-preprocessed event stream, which increased recall, precision, the percentage overlap, and the intersection over union of the car bounding boxes.

% or


\ifCLASSOPTIONcompsoc
  \section*{Acknowledgments}
\else
  \section*{Acknowledgment}
\fi

This work was made possible in part by funding from Ford Motor Company University Research Program.

\ifCLASSOPTIONcaptionsoff
  \newpage
\fi



%\bibliographystyle{IEEEtran}

{\small
\bibliographystyle{IEEEtran}
\bibliography{report}
}


\begin{IEEEbiography}{Abigail Wolf}
Biography text here.
\end{IEEEbiography}

\begin{IEEEbiography}{Shannon Brooks-Lehnert}
Biography text here.
\end{IEEEbiography}

\begin{IEEEbiography}{Keigo Hirakawa}
Biography text here.
\end{IEEEbiography}

%\newpage


%\vfill


\end{document}